\documentclass[conference, 10pt, twocolumn]{ieeeconf}       
\bstctlcite{bstctl:etal, bstctl:nodash, bstctl:simpurl}
\def\BibTeX{{\rm B\kern-.05em{\sc i\kern-.025em b}\kern-.08em
    T\kern-.1667em\lower.7ex\hbox{E}\kern-.125emX}}
    
\usepackage[left=54pt, right=54pt,bottom=54pt, top=54pt]{geometry}

\usepackage{amsmath,mathrsfs,amsfonts,amssymb,graphicx,epsfig}
 \usepackage{amsthm}
\usepackage{subcaption}
\usepackage{color,multirow,rotating,tcolorbox}
\usepackage{algorithm,algpseudocode,algorithmicx}
\usepackage{cite,url,framed,bm,balance,nicematrix,physics}
\usepackage{stmaryrd,mathtools}

\newtheorem{theorem}{Theorem}
\newtheorem{corollary}[theorem]{Corollary}

\newtheorem{proposition}{Proposition}
\newtheorem{remark}{Remark}

\newtheorem{example}{Example}

\newtheorem{mytheorem}{Theorem}[]      
\newenvironment{mycolortheorem}[1][]{
   \begin{tcolorbox}[colback=blue!5!white,
      width=\dimexpr\linewidth\relax,
      enlarge left by=0pt,
      enlarge right by=-5pt,
      boxsep=5pt,
      left=0pt,
      right=0pt,
      top=0pt,
      bottom=0pt,
      arc=0pt,
      boxrule=0pt,
      colframe=white]{}{}
      \ifstrempty{#1}{
         \begin{mytheorem}
      }{
         \begin{mytheorem}[#1]
      }%
}{%
      \end{mytheorem}
   \end{tcolorbox}%
}

\usepackage{tikz}
\usetikzlibrary{calc,trees,positioning,arrows,chains,shapes.geometric,decorations.pathreplacing,decorations.pathmorphing,shapes, matrix,shapes.symbols}

\usepackage[breaklinks]{hyperref}

\makeatletter
\newcommand{\pushright}[1]{\ifmeasuring@#1\else\omit\hfill$\displaystyle#1$\fi\ignorespaces}

\newcommand{\dotminus}{\mathbin{\text{\@dotminus}}}

\newcommand{\@dotminus}{%
  \ooalign{\hidewidth\raise1ex\hbox{.}\hidewidth\cr$\m@th-$\cr}%
}

\allowdisplaybreaks

\title{\LARGE\textbf{Optimal Multimarginal Schr\"{o}dinger Bridge: Minimum Spanning Tree over Measure-valued Vertices
}}

\author{Georgiy A. Bondar, and Abhishek Halder
\thanks{Georgiy A. Bondar is with the Department of Applied Mathematics, University of California Santa Cruz, CA 95064, USA, {\tt\footnotesize gbondar@ucsc.edu}.}
\thanks{Abhishek Halder is with the Department of Aerospace Engineering, Iowa State University, Ames, IA 50011, USA, {\tt\footnotesize{ahalder@iastate.edu}}.}
\thanks{This research was partially supported by NSF award 2111688.}
}

\IEEEoverridecommandlockouts
\begin{document}
\bstctlcite{IEEEexample:BSTcontrol}
\maketitle
\thispagestyle{empty}
\pagestyle{empty}

\begin{abstract}
    The Multimarginal Schr\"{o}dinger Bridge (MSB) finds the optimal coupling among a collection of random vectors with known statistics and a known correlation structure. In the MSB formulation, this correlation structure is specified \emph{a priori} as an undirected connected graph with measure-valued vertices. In this work, we formulate and solve the problem of finding the optimal MSB in the sense we seek the optimal coupling over all possible graph structures. We find that computing the optimal MSB amounts to solving the minimum spanning tree problem over measure-valued vertices. We show that the resulting problem can be solved in two steps. The first step constructs a complete graph with edge weight equal to a sum of the optimal value of the corresponding bimarginal SB and the entropies of the endpoints. The second step solves a standard minimum spanning tree problem over that complete weighted graph. Numerical experiments illustrate the proposed solution.   
\end{abstract}

\section{Introduction}\label{sec:Intro}
\noindent\textbf{Multimarginal Schr\"{o}dinger bridge (MSB).} The (graph-structured) MSB is a probabilistic generative model with maximum likelihood optimality guarantee. Specifically, consider a known collection of $s\in\mathbb{N}_{\geq 2}$ measure-valued snapshots, i.e., probability vectors $\bm{\mu}_{1}\in\Delta^{n_{1} - 1}, \hdots, \bm{\mu}_{s}\in\Delta^{n_{s} - 1}$, where the probability simplex $\Delta^{n_{\sigma}-1}:=\{\bm{x}\in\mathbb{R}^{n_{\sigma}}_{\geq 0}\mid \langle\bm{1},\bm{x}\rangle = 1\}$, $n_{\sigma}\in\mathbb{N}$ $\forall\sigma\in \llbracket s\rrbracket:=\{1,\hdots,s\}$, the symbol $\bm{1}$ denotes the all-ones vector, and $\langle\cdot,\cdot\rangle$ is the Hilbert-Schmidt inner product.

Now, take these snapshots to be the vertex set $\mathcal{V}:=\{\bm{\mu}_\sigma\}_{\sigma\in\llbracket s\rrbracket}$ of an \emph{undirected connected} graph $\mathcal{G}=(\mathcal{V},\mathcal{E})$, where $\mathcal{E}$ is the edge set. With $\otimes$ denoting the tensor product, let $\Pi$ be the set of all coupling tensors with marginals $\mathcal{V}$, i.e.,
\begin{align}
\Pi(\mathcal{V})\!\! = &\Pi\left(\bm{\mu}_{1},\hdots,\bm{\mu}_{s}\right) := \{\bm{M}\in\mathbb{R}^{n_{1}}_{\geq 0}\otimes\hdots\otimes\mathbb{R}^{n_{s}}_{\geq 0}\mid \nonumber\\
&\qquad\qquad\qquad{\mathrm{proj}}_{\sigma}\left(\bm{M}\right)=\bm{\mu}_{\sigma}\:\forall\sigma\in\llbracket s\rrbracket\},
\label{defPi}    
\end{align}
where ${\mathrm{proj}}_{\sigma}:\mathbb{R}^{n_{1}}_{\geq 0}\otimes\hdots\otimes\mathbb{R}^{n_{s}}_{\geq 0}\mapsto\mathbb{R}^{n_{\sigma}}_{\geq 0}$, and is given componentwise as
\begin{align}
\left(\operatorname{proj}_\sigma(\boldsymbol{M})\right)_r=\!\!\!\!\sum_{i_1, \ldots, i_{\sigma-1}, i_{\sigma+1}, \ldots, i_s} \!\!\!\!\!\!\boldsymbol{M}_{i_1, \ldots, i_{\sigma-1}, r, i_{\sigma+1}, \ldots, i_s}.
\label{defProj}    
\end{align}
In \eqref{defProj} and throughout, we use the square braces $\left[\cdot\right]$ to denote components of tensors (boldfaced capital letters) and matrices (unboldfaced capital), and parentheses $(\cdot)$ to denote components of vectors (boldfaced lowercase).

To compare candidate couplings $\bm{M}$, we define a ground cost tensor $\bm{C}\in\mathbb{R}^{n_{1}}_{\geq 0}\otimes\hdots\otimes\mathbb{R}^{n_{s}}_{\geq 0}$ that decouples along the edges of $\mathcal{G}$. Intuitively, entries of $\bm{C}$ encode the costs of transporting unit amount of mass along the edges of $\mathcal{G}$. 

To exemplify the dependence of the ground cost tensor $\bm{C}$ on the graph structure, consider when $\mathcal{G}$ is a \emph{path}. Then $\sigma$ indexes a direction (e.g., of time, a spatial dimension) and
\begin{align}
[\bm{C}_{i_\mathcal{V}}] = [\bm{C}_{i_1,\dots,i_s}] = \displaystyle\sum_{\sigma=1}^{s-1}[(C_{\sigma})_{i_\sigma,i_{\sigma+1}}]\:,
\label{CforPath}    
\end{align}
where the matrices $C_{\sigma}$ are given componentwise as $[(C_{\sigma})_{i_\sigma,i_{\sigma+1}}] := c_{\sigma}\left(\bm{x}_{\sigma},\bm{x}_{\sigma + 1}\right)$, a pairwise cost between the random vectors $\bm{x}_{\sigma}\sim\bm{\mu}_{\sigma}$ and $\bm{x}_{\sigma+1}\sim\bm{\mu}_{\sigma+1}$. 
As another example, when $\mathcal{G}$ is a \emph{star}, with one vertex $\mu_1$ called the \emph{barycenter} and all other vertices connected to and only to the barycenter, the entries of the ground cost tensor $\bm{C}$ are 
\begin{align}
[\bm{C}_{i_\mathcal{V}}] = \displaystyle\sum_{\sigma=2}^{s}[(C_{\sigma})_{i_1,i_{\sigma}}]\:,
\label{CforBarycenter}    
\end{align}
where $C_\sigma$ contains the pairwise costs between $\bm{x}_1\sim\bm{\mu}_1$ and $\bm{x}_{\sigma}\sim\bm{\mu}_\sigma$. Paths appear naturally in tracking problems while stars appear in information fusion.

For a general graph $\mathcal{G}=(\mathcal{V},\mathcal{E})$, 
\begin{align}
[\bm{C}_{i_\mathcal{V}}] = \displaystyle\sum_{(\sigma_1,\sigma_2)\in\mathcal{E}}[(C_{\sigma_1\sigma_2})_{i_{\sigma_1},i_{\sigma_2}}]\:.
\label{CforGraph}    
\end{align}

In general, $\mathcal{X}_{\sigma}:={\mathrm{support}}(\bm{\mu}_{\sigma})\subseteq\mathbb{R}^{d}$ need not be identical for all $\sigma\in\llbracket s\rrbracket$. Let 
\begin{align}
\bm{\mathcal{X}}:=\displaystyle\prod_{\sigma\in\llbracket s\rrbracket}\mathcal{X}_{\sigma} \subseteq\left(\mathbb{R}^{d}\right)^{\otimes s}.
\label{ProductGroundSpace}    
\end{align}
Then $\bm{C}:\bm{\mathcal{X}}\mapsto\mathbb{R}^{n_{1}}_{\geq 0}\otimes\hdots\otimes\mathbb{R}^{n_{s}}_{\geq 0}$. Likewise, the couplings $\bm{M}$ are supported on $\bm{\mathcal{X}}$. 

With a fixed regularization parameter $\eta>0$, the MSB is the optimal coupling
\begin{align}
\bm{M}^{\mathrm{opt}}:=\underset{\bm{M}\in\Pi\left(\mathcal{V}\right)}{\arg\min}\quad\langle \bm{C}+\eta\log\bm{M},\bm{M}\rangle,
\label{defMopt}    
\end{align}
i.e., the minimizer of the entropy\footnote{Specifically the \emph{Shannon entropy} $H(\bm{M}):=-\langle\log\bm{M},\bm{M}\rangle$.}-regularized multimarginal optimal transport \cite{pass2015multi} problem over graph structure $\mathcal{G}$.

Letting $\bm{K}:=\exp\left(-\bm{C}/\eta\right)$ where $\exp$ acts elementwise, notice that \eqref{defMopt} can be expressed as
\begin{align}
\bm{M}^{\mathrm{opt}}=\underset{\bm{M}\in\Pi\left(\mathcal{V}\right)}{\arg\min}\,\eta\:
{\mathrm{D_{KL}}}\left(\bm{M}\parallel \bm{K}\right),
\label{eq:minWKLdiv}    
\end{align}
where ${\mathrm{D_{KL}}}$ denotes the \emph{relative entropy} a.k.a. the \emph{Kullback-Leibler divergence}\footnote{By definition, ${\mathrm{D_{KL}}}(\bm{P}\parallel\bm{Q})=\langle\bm{P},\log\left(\bm{P}\oslash\bm{Q}\right)\rangle$.}.
Re-writing \eqref{defMopt} in the form \eqref{eq:minWKLdiv} has two merits. First, it clarifies that the optimal coupling $\bm{M}^{\mathrm{opt}}$ is the \emph{most likely joint consistent with the given $s$ marginals}. Second, it explains the \emph{existence-uniqueness} of $\bm{M}^{\mathrm{opt}}$ as \eqref{eq:minWKLdiv} involves a strictly convex program: a Kullback-Leibler projection onto a convex polyhedron\footnote{intersection of the $s$ hyperplanes in \eqref{defProj} with a simplex in dimension $\prod_{\sigma\in\llbracket s\rrbracket}n_{\sigma}$.} $\Pi$.

Following standard nomenclature, the bimarginal ($s=2$) case is called the \emph{Schr\"{o}dinger bridge (SB)} and the term MSB is reserved for $s\geq 3$.

\noindent\textbf{Related works.} Recent interest in MSB is due to its growing use for learning population-level trajectories from snapshot data \cite{chen2023deep,shen2025multi}. Applications include sensor fusion \cite{elvander2020multi}, tracking an ensemble of agents \cite{haasler2021multimarginal}, image interpolation \cite{noble2023tree}, trajectory inference in single-cell RNA sequencing \cite{berlinghieri2025oh}, and learning computational resource usage \cite{bondar2024path,bondar2025stochastic}. For a control-theoretic exposition to MSB, see \cite{chen2019multi}.

The standard approach to solve \eqref{defMopt} (equivalently \eqref{eq:minWKLdiv}) is to use the convergent multimarginal Sinkhorn recursion \cite{benamou2015iterative,marino2020optimal,carlier2022linear} that computes
\begin{subequations}
\begin{align}
&\bm{u}_{\sigma} := \exp\left(-\bm{\lambda}_{\sigma}/\eta\right),\quad \bm{U} = \otimes_{\sigma\in\llbracket s\rrbracket}\bm{u}_{\sigma}, \label{defU}\\
& \boldsymbol{u}_{\sigma} \leftarrow \boldsymbol{u}_{\sigma} \odot \boldsymbol{\mu}_{\sigma} \oslash \operatorname{proj}_{\sigma}(\boldsymbol{K} \odot \boldsymbol{U}) \quad\forall \sigma\in \llbracket s\rrbracket. \label{MultiSinkRecursion}    
\end{align}
\label{MultimargSinkhornSolution}
\end{subequations}
In \eqref{defU}, the $\exp$ is elementwise, and the vectors $\bm{\lambda}_{\sigma}\in\mathbb{R}^{n_{\sigma}}$ are the Lagrange multipliers associated with the equality constraints in \eqref{defPi}. In \eqref{MultiSinkRecursion}, the symbols $\odot$ and $\oslash$ denote elementwise (Hadamard) multiplication and division, respectively. In particular, \eqref{MultimargSinkhornSolution} has guaranteed convergence with worst-case linear rate. The updates \eqref{MultiSinkRecursion} can be cyclic, greedy or randomized \cite{lin2022complexity} across the index $\sigma\in\llbracket s\rrbracket$.

The optimal primal variable $\bm{M}^{\mathrm{opt}}$ is then computed in terms of the converged $\bm{U}$ from the recursion \eqref{MultimargSinkhornSolution} as
\begin{align}
\bm{M}^{\mathrm{opt}}=\bm{K}\odot\bm{U}.
\label{MoptfromU}    
\end{align}

The main computational challenge in solving the MSB problems lies in evaluating the projection in \eqref{MultiSinkRecursion}, which in general is known \cite{lin2022complexity} to be exponential in $s$, the number of vertices. See also \cite{altschuler2023polynomial}. However, when $\mathcal{G}$ is a tree, this computational complexity becomes linear in $s$. This is because the projection for tree-structured $\mathcal{G}$ can be computed as message passing \cite{haasler2021multi} via belief propagation \cite{yedidia2003understanding}. 

For general connected $\mathcal{G}$ and $n_{\sigma}=n$ $\forall\sigma\in\llbracket s\rrbracket$, the work in \cite{fan2022complexity} derives a complexity $\widetilde{\mathcal{O}}\left({\mathrm{diam}}(\mathcal{T})sn^{w(\mathcal{G})+1}\varepsilon^{-2}\right)$ for an $\varepsilon$-accurate solution when $\mathcal{G}$ can be factored as a junction tree $\mathcal{T}$ with diameter ${\mathrm{diam}}(\mathcal{T})$ and treewidth\footnote{i.e., width of the junction tree decomposition} $w(\mathcal{G})$. 

In summary, existing MSB literature focuses on a fixed $\mathcal{G}$, and computational complexity depends on the structure of $\mathcal{G}$.  

\noindent\textbf{Optimal graph structure.} Different from the existing MSB literature, we think of $\mathcal{G}$ as a variable, and consider optimizing the MSB cost among all connected $\mathcal{G}$ over the $s$ given measure-valued vertices. 

Since $\bm{C}=\bm{C}\left(\mathcal{G}\right)$, $\bm{K}=\bm{K}\left(\mathcal{G}\right)$, $\bm{U}=\bm{U}\left(\mathcal{G}\right)$, and thus $\bm{M}^{\mathrm{opt}} = \bm{M}^{\mathrm{opt}}\left(\mathcal{G}\right)$, seeking the \emph{optimal graph structure} $\mathcal{G}^{\mathrm{opt}}$ amounts to solving the variational problem:
\begin{align}
\hspace*{-0.1in}&\mathcal{G}^{\mathrm{opt}} \!= \!\!\!\!\!\underset{\substack{\mathcal{G}\:\text{undirected connected}\\
\text{over vertices}\: \mathcal{V}}}{\arg\min}\!\!\langle\bm{C}(\mathcal{G})+\eta\log\bm{M}^{\mathrm{opt}}\left(\mathcal{G}\right),\bm{M}^{\mathrm{opt}}\left(\mathcal{G}\right)\rangle\nonumber\\
&=\underset{\substack{\mathcal{G}\:\text{undirected connected}\\
\text{over vertices}\: \mathcal{V}}}{\arg\min}\:\underset{\bm{M}\in\Pi\left(\mathcal{V}\right)}{\min}\langle \bm{C}(\mathcal{G})+\eta\log\bm{M},\bm{M}\rangle.
\label{OptimalGproblem}    
\end{align}
We refer to the \emph{optimal coupling} associated with $\mathcal{G}^{\mathrm{opt}}$ as the \emph{optimal MSB}. We focus on the MSB (i.e., $s\geq 3$) since for $s=2$, outer minimization in \eqref{OptimalGproblem} is trivial as then the line graph is the only possible connectivity.

\begin{figure}[t]
    \centering
\includegraphics[width=0.68\linewidth]{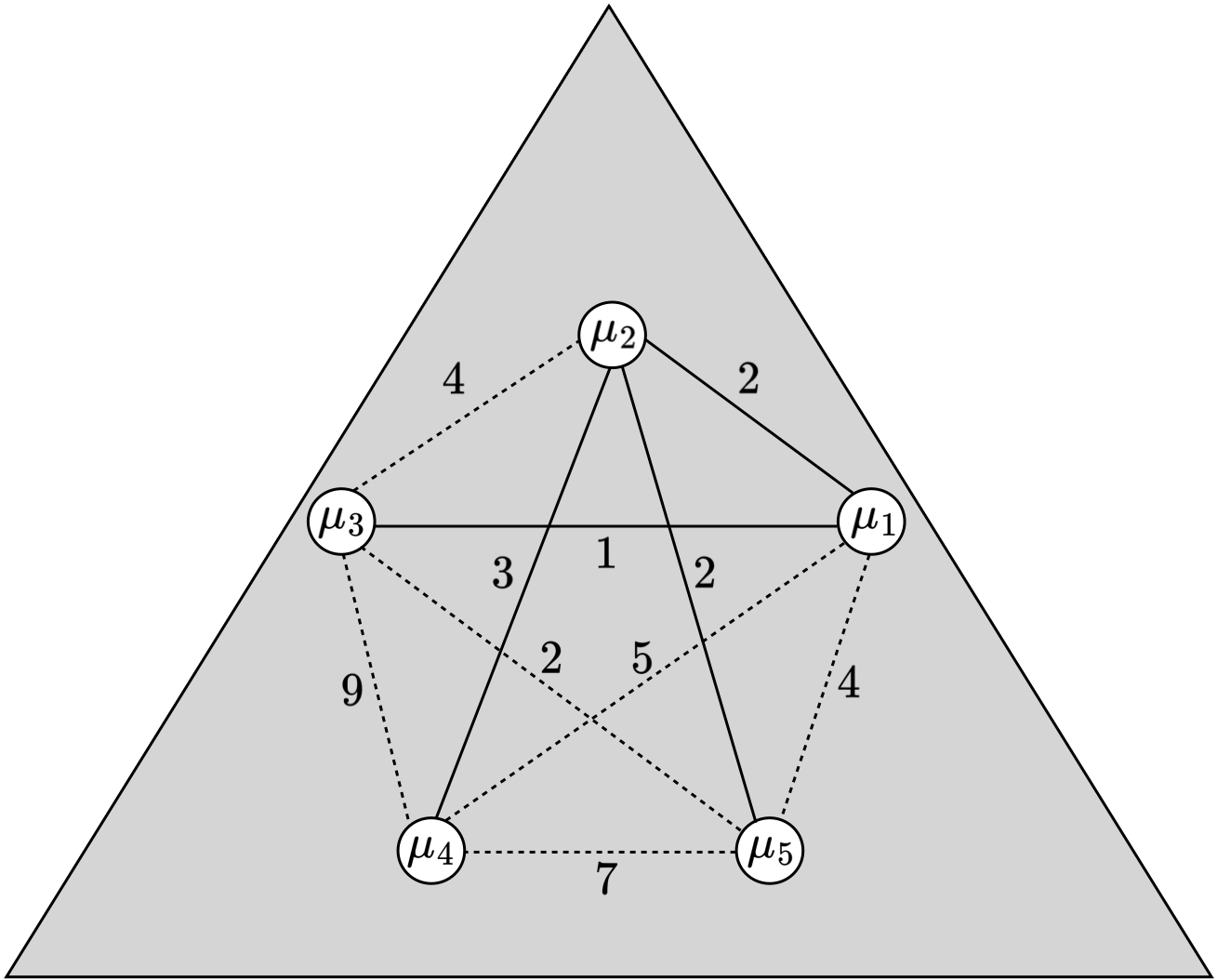}
    \caption{The MST (solid lines) of a complete weighted undirected graph with $s=5$ measure-valued vertices, where the measures have identical supports $\mathcal{X}_{\sigma}$, and $n_{\sigma} = 3$ $\forall\sigma\in\llbracket s=5\rrbracket$. The edge-weights are shown as numerical values along the edges. We will derive the edge weights in Sec. \ref{subsec:MSTconstrcution}. The edges not included in the MST are shown as dashed lines. The shaded triangle is the probability simplex $\Delta^{2}$. For non-identical $\mathcal{X}_{\sigma}$, the MST is inter-simplex instead of intra-simplex as shown here.}
    \vspace*{-0.2in}
    \label{fig:MSTonK5}
\end{figure}

\noindent\textbf{Motivation.}  In practice, the graph $\mathcal{G}$ may not be known \emph{a priori}. As a motivating example, consider different images of a natural disaster (e.g., wildfire) or a sports event taken at different and possibly unknown times, from different locations, pose, illumination, and a combination thereof. In such scenarios, the \emph{most-likely spatio-temporal reconstruction} can naturally be posed as an MSB with the caveat that the correlation graph structure and the amount of correlation need to be co-optimized.


\noindent\textbf{Contributions.} Our specific contributions are twofold.
\begin{itemize}
\item We introduce the problem of finding the optimal MSB, i.e., the problem of finding the optimal connected graph structure w.r.t. the MSB cost over a given set of measure-valued vertices. We deduce that the optimal graph structure is a spanning tree, thereby arriving at a novel tree optimization problem \cite{magnanti1995optimal}.

\item We show that the problem can be solved in two steps: first by constructing a complete graph with the edge weights being a sum of the optimal values of the bimarginal SB between the corresponding vertices and the endpoint entropies, and then finding a minimum spanning tree (MST) over that complete graph (see Fig. \ref{fig:MSTonK5}). For $n_{\sigma}=n$ $\forall\sigma\in\llbracket s\rrbracket$, we are then able to solve \eqref{OptimalGproblem} with time complexity $\mathcal{O}\left(s^2 n^2\|\bm{C}\|_{\infty}^{2}\left(\log n\right)^{-1}/\eta^2\right)$ using standard MST algorithms such as Dijkstra-Jarník-Prim algorithm or Borůvka's algorithm \cite[Sec. 2]{pettie2002optimal}.
\end{itemize}

\noindent\textbf{Organization.} 
In Sec. \ref{sec:MST}, we explain that the $\mathcal{G}^{\mathrm{opt}}$ in \eqref{OptimalGproblem} is an MST w.r.t. the MSB cost. In Sec. \ref{sec:TwoStepSolution}, we show a tree decomposition property that makes computing the same tractable. Numerical results are reported in Sec. \ref{sec:NumericalResults} followed by concluding remarks in Sec. \ref{sec:Conclusion}.


\section{Optimal MSB and Minimum Spanning Tree}\label{sec:MST}
We start with the following observation.
\begin{proposition}\label{prop:Goptisatree}
The $\mathcal{G}^{\rm{opt}}$ in \eqref{OptimalGproblem} is a spanning tree $\mathcal{T}^{\rm{opt}}$ of the complete graph over the $s$ vertices in $\mathcal{V}=\{\bm{\mu}_1, \hdots, \bm{\mu}_s\}$. 
\end{proposition}
\begin{proof}
As $\mathcal{G}^{\rm{opt}}$ is connected, it is sufficient to show that $\mathcal{G}^{\rm{opt}}$ does not contain cycles.

Assume instead that $\mathcal{G}^{\rm{opt}}=(\mathcal{V},\mathcal{E})$ contains a cycle. Then there exists an edge connecting $(\sigma',\sigma'')\in\mathcal{V}\times\mathcal{V}$ which may be removed without breaking connectivity. Call the graph with this edge removed as $\mathcal{G}'$. It follows from \eqref{CforGraph} that
$$ [\bm{C}(\mathcal{G}')_{i_\mathcal{V}}] = \!\!\!\!\sum_{(\sigma_1,\sigma_2)\in\mathcal{E}\setminus(\sigma',\sigma'')}[(C_{\sigma_1\sigma_2})_{i_{\sigma_1},i_{\sigma_2}}] < [\bm{C}(\mathcal{G}^{\rm{opt}})_{i_\mathcal{V}}].$$
So the objective value of \eqref{OptimalGproblem} will be lower for $\mathcal{G}'$, a contradiction.
\end{proof}
As a consequence of Proposition \ref{prop:Goptisatree}, it suffices to consider the $\arg\min$ in \eqref{OptimalGproblem} over \emph{the set of spanning trees of the complete graph over the $s$ vertices $\bm{\mu}_1, \hdots, \bm{\mu}_s$}. In other words, solving \eqref{OptimalGproblem} amounts to finding an MST w.r.t. the MSB cost. This problem in itself is interesting because it opens up the possibility to generalize the Euclidean MST \cite{shamos1975closest,steele1988growth,agarwal1990euclidean} and their applications to situations where only the statistics of the vertex locations, as opposed to their exact locations, are available.  

However, finding the optimal tree-structured MSB is nontrivial since by Cayley's theorem \cite{cayley1878theorem}, the number of spanning trees over $s$ labeled vertices is $s^{s-2}$. So it is impractical to first solve all the corresponding MSBs and then determining the minimizing tree structure. In Sec. \ref{sec:TwoStepSolution}, we will show a tree decomposition property enabling tractable solution. 

At this point, it is natural to wonder if $\mathcal{T}^{\rm{opt}}$ might be of special kind such as a \emph{path}. For $s=3$, this is indeed the case since all spanning trees over three vertices are paths. For $s>3$, the MST need not be a path, as shown in Example \ref{Example:PathsNotMST}.

\begin{example}\label{Example:PathsNotMST}
Consider the $s=5$ Dirac delta measures shown in Fig. \ref{fig:subopt_pathtree} with the Euclidean ground cost. The only feasible coupling is a Dirac measure that equals to the product of the given Dirac measures. So the entropy term in \eqref{defMopt} vanishes and it is easy to verify that $\mathcal{T}^{\rm{opt}}$ is a star--not a path--with solid edges shown in Fig. \ref{fig:subopt_pathtree}. 
\end{example}

\begin{figure}
    \centering
    \includegraphics[width=0.5\linewidth]{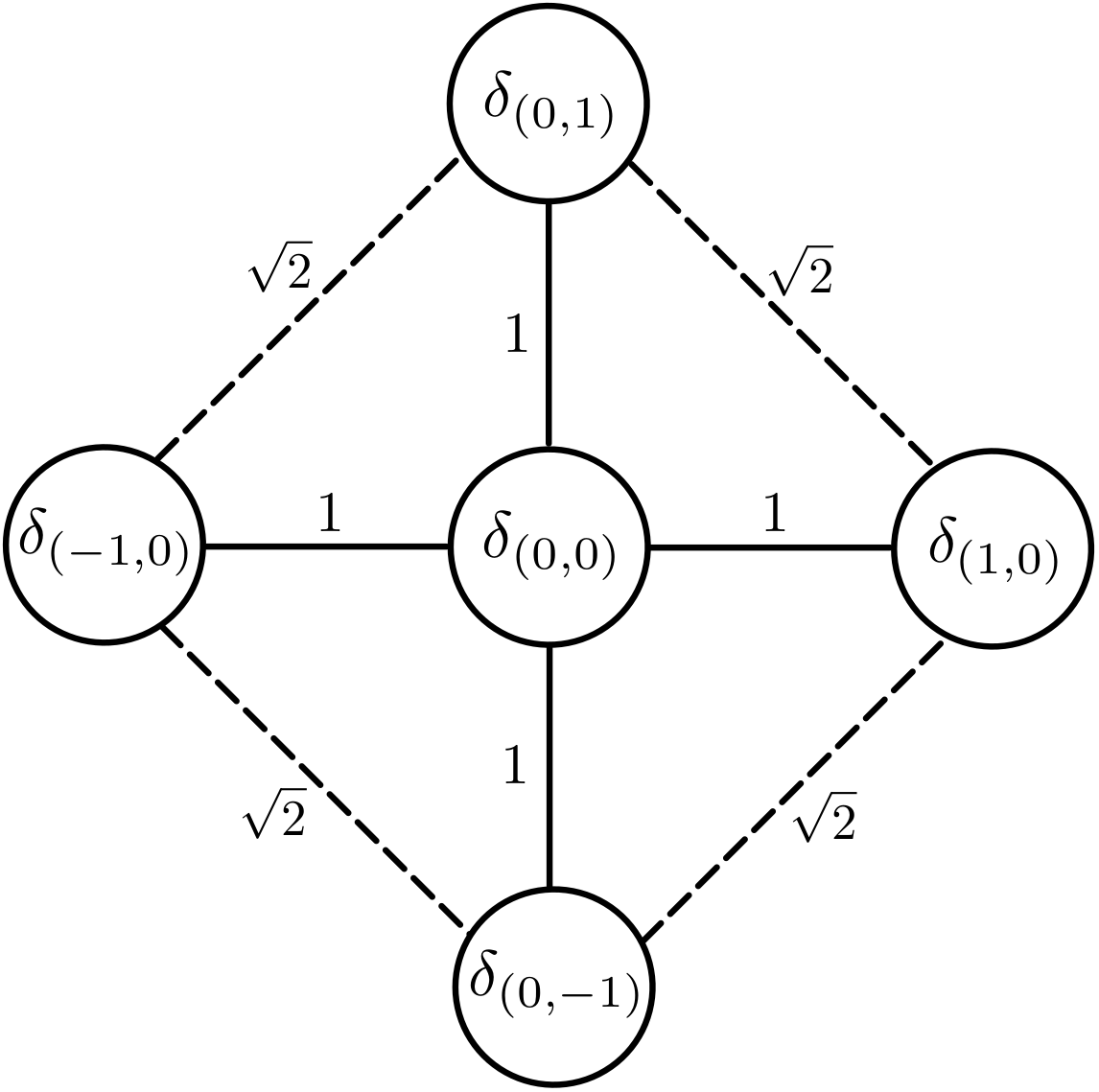}
    \caption{With the $s=5$ Dirac delta measures $\mu_{\sigma}=\delta_{\bm{x}_{\sigma}}$, $\bm{x}_{\sigma}\in\mathbb{R}^2\:\forall\sigma\in\llbracket s\rrbracket$ as above, and with the Euclidean distance as the ground cost, the $\mathcal{T}^{\rm{opt}}$ (with solid edges) is a star.}
    \vspace*{-0.2in}
    \label{fig:subopt_pathtree}
\end{figure}


\section{Solution of the Optimal MSB Problem}\label{sec:TwoStepSolution}
As in the classical MST problem, the intractability of exhaustive search motivates constructive or greedy algorithms (e.g., Dijkstra-Jarník-Prim or Borůvka's algorithm \cite[Sec. 2]{pettie2002optimal}), which sequentially build the MST edge-by-edge. Such approaches are enabled by the property that the edge weights are known and additive. However, it is by no means obvious that such a property holds for our problem \eqref{OptimalGproblem}. In this Section we establish an analogous decomposition property. We explain how this property helps to express \eqref{OptimalGproblem} as an instance of the classical MST problem.


\subsection{Tree Decomposition}\label{subsec:treedecomposition}
Given a tree $\mathcal{T}=(\mathcal{V},\mathcal{E})$ over the vertex set $\mathcal{V}=\{\bm{\mu}_\sigma\}_{\sigma\in\llbracket s\rrbracket}$, we would like to evaluate the objective of \eqref{OptimalGproblem} -- i.e., the tree-structured MSB -- as a sum over the tree's constituent edges. A key result of Haasler et al. \cite[Prop. 3.4]{haasler2021multimarginal} allows for the `cutting' of such trees into independent subtrees. In the following proposition, we establish its constructive version, which makes this decomposition explicit.

\begin{proposition}
\label{prop:indep_bimarginal_M}
    Let $\mathcal{T}=(\mathcal{V},\mathcal{E})$ be a tree as above, and let $\bm{\mu}_\sigma\in\mathcal{V}$ be a non-leaf which splits $\mathcal T$ into two subtrees $\mathcal{T}_1=(\mathcal{V}_1,\mathcal{E}_1)$ and $\mathcal{T}_2=(\mathcal{V}_2,\mathcal{E}_2)$, joined at $\bm{\mu}_\sigma$. Let
    \begin{equation}\label{eq:minWKLdiv_tree}
        \bm{M}^{\rm{opt}}_{\mathcal{T}} := \underset{\bm{M}\in\Pi(\mathcal{V})}{\arg\min}\:\:\eta\:   {\mathrm{D_{KL}}}\left(\bm{M}\parallel \bm{K}\right).
    \end{equation}
    Then, we have the decomposition
\begin{equation}\label{eq:M_decomp}
        [\left(\bm{M}^{\rm{opt}}_{\mathcal{T}}\right)_{i_\mathcal{V}}] = \frac{[(\bm{M}^{\rm{opt}}_{\mathcal{T}_1})_{i_{\mathcal{V}_1}}][(\bm{M}^{\rm{opt}}_{\mathcal{T}_2})_{i_{\mathcal{V}_2}}]}{(\bm{\mu}_\sigma)_{i_\sigma}}.
    \end{equation}
    Furthermore, for $k\in\{1,2\}$, letting 
    \begin{align}
[(\bm{K}_k)_{i_{\mathcal{V}_k}}] := \!\!\!\!\prod_{(\sigma_1,\sigma_2)\in\mathcal{E}_k}\!\!\!\!\!\!\!\exp(-[(C_{\sigma_1,\sigma_2})_{i_{\sigma_1},i_{\sigma_2}}]/\eta), 
    \label{ComponentsOfKk}    
    \end{align}
    we have $ [\bm{K}_{i_{\mathcal{V}}}] = [(\bm{K}_1)_{i_{\mathcal{V}_1}}]\cdot[(\bm{K}_2)_{i_{\mathcal{V}_2}}]$, and
    \begin{align}\label{eq:KL_decomp}
        &{\mathrm{D_{KL}}}\!\left(\bm{M}^{\rm{opt}}_{\mathcal{T}}\!\parallel\! \bm{K}\right) \!=\! \!\sum_{k=1,2}\!{\mathrm{D_{KL}}}\!\left(\bm{M}^{\rm{opt}}_{\mathcal{T}_k}\!\parallel\! \bm{K}_k\right) + H(\bm{\mu}_\sigma). 
    \end{align}
\end{proposition}
\begin{proof}

For any $\bm{M}\in\Pi\left(\mathcal{V}\right)$, by chain rule of probability, 
\begin{align*}
    [\bm{M}_{i_{\mathcal{V}}}] &= \underbrace{\frac{[{\rm{proj}}_{\mathcal{V}_1}(\bm{M})_{i_{\mathcal{V}_1}}]}{(\bm{\mu}_\sigma)_{i_\sigma}}}_{:=[(\bm{M}_1)_{i_{\mathcal{V}_1}}]}\cdot(\bm{\mu}_\sigma)_{i_\sigma}\cdot\underbrace{\frac{[{\rm{proj}}_{\mathcal{V}_2}(\bm{M})_{i_{\mathcal{V}_2}}]}{(\bm{\mu}_\sigma)_{i_\sigma}}}_{:=[(\bm{M}_2)_{i_{\mathcal{V}_2}}]}. 
\end{align*}
Now, letting $\mathcal{S}_1:=\mathcal{V}_1\setminus\{\bm{\mu}_\sigma\}$ and $\mathcal{S}_2:=\mathcal{V}_2\setminus\{\bm{\mu}_\sigma\}$, we have
\begin{align*}
    &{\mathrm{D_{KL}}}(\bm{M} \parallel \bm{K} ) := \sum_{i_{\mathcal{V}}} [\bm{M}_{i_{\mathcal{V}}}]\log\frac{[\bm{M}_{i_{\mathcal{V}}}]}{[\bm{K}_{i_{\mathcal{V}}}]} \\
    &= \sum_{i_{\mathcal{S}_1}}\sum_{i_{\sigma1}}(\bm{\mu}_\sigma)_{i_\sigma}[(\bm{M}_2)_{i_{\mathcal{V}_2}}]\!\cdot\!\log\!\frac{[(\bm{M}_1)_{i_{\mathcal{V}_1}}](\bm{\mu}_\sigma)_{i_\sigma}[(\bm{M}_2)_{i_{\mathcal{V}_2}}]}{[(\bm{K}_1)_{i_{\mathcal{V}_1}}][(\bm{K}_2)_{i_{\mathcal{V}_2}}]} \\
    &= \!\!\sum_{i_\sigma}(\bm{\mu}_\sigma)_{i_\sigma}\!\log(\bm{\mu}_\sigma)_{i_\sigma}\!\!\underbrace{\sum_{i_{\mathcal{S}_1}}[(\bm{M}_1)_{i_{\mathcal{V}_1}}]}_{=1}\!\underbrace{\sum_{i_{\mathcal{S}_2}}[(\bm{M}_2)_{i_{\mathcal{V}_2}}]}_{=1} \!+\alpha_1 \!+\! \alpha_2
\end{align*}
where for $k=\{1,2\}$, 
\begin{align*}
    \alpha_k &:= \sum_{i_\sigma}(\bm{\mu}_\sigma)_{i_\sigma}\sum_{i_{S_k}}[(\bm{M}_k)_{i_{\mathcal{V}_k}}]\log\frac{[(\bm{M}_k)_{i_{\mathcal{V}_k}}]}{[(\bm{K}_k)_{i_{\mathcal{V}_k}}]} \\
    &= \sum_{i_{\mathcal{V}_k}}[{\rm{proj}}_{\mathcal{V}_k}(\bm{M})_{i_{\mathcal{V}_k}}]\log\frac{[{\rm{proj}}_{\mathcal{V}_k}(\bm{M})_{i_{\mathcal{V}_k}}]}{[(\bm{K}_k)_{i_{\mathcal{V}_k}}]} \\
    &\;- \sum_{i_\sigma}(\bm{\mu}_\sigma)_{i_\sigma}\log(\bm{\mu}_\sigma)_{i_\sigma}\sum_{i_{S_k}}[(\bm{M}_k)_{i_{\mathcal{V}_k}}] \\
    &\!\!\!\!\!\!\!\!\!\!\!\!\!\!= {\mathrm{D_{KL}}}({\rm{proj}}_{\mathcal{V}_k}\!(\bm{M}) \!\parallel\! \bm{K}_k ) \!-\! \sum_{i_\sigma}(\bm{\mu}_\sigma)_{i_\sigma}\!\log(\bm{\mu}_\sigma)_{i_\sigma}.
\end{align*}
Thus by rewriting the constraint set of \eqref{eq:minWKLdiv_tree} as
$$ \Pi(\mathcal{V}) = \Bigg\{ \bm{M}\in(\mathbb{R}_{\geq 0}^n)^{\otimes s} \:\bigg\rvert\: \begin{aligned}
    {\rm{proj}}_{\mathcal{V}_1}(\bm{M}) &\in \Pi(\mathcal{V}_1),\\
    {\rm{proj}}_{\mathcal{V}_2}(\bm{M}) &\in \Pi(\mathcal{V}_2)
\end{aligned}  \Bigg\}\:, $$
the minimization \eqref{eq:minWKLdiv_tree} becomes
\begin{equation}\label{eq:minWKLdiv_tree_decoupled}
    \underset{\substack{{\rm{proj}}_{\mathcal{V}_1}(\bm{M})\in\Pi(\mathcal{V}_1) \\ {\rm{proj}}_{\mathcal{V}_2}(\bm{M})\in\Pi(\mathcal{V}_2)}}{\arg\min} \eta\!\left(\begin{aligned}
        & {\mathrm{D_{KL}}}({\rm{proj}}_{\mathcal{V}_1}\!(\bm{M}) \!\parallel\! \bm{K}_1) \\
        & + {\mathrm{D_{KL}}}({\rm{proj}}_{\mathcal{V}_2}\!(\bm{M}) \!\parallel\! \bm{K}_2) + H(\bm{\mu}_\sigma)
    \end{aligned}\right)\:.
\end{equation}
Note that the objective of \eqref{eq:minWKLdiv_tree_decoupled} is a separable sum, allowing the decoupling of the minimization. Herefrom both \eqref{eq:M_decomp} and \eqref{eq:KL_decomp} follow.
\end{proof}

Proposition \ref{prop:indep_bimarginal_M} allows for the solution of a given tree-structured MSBP as the combination (by \eqref{eq:M_decomp}) of its solution over any two subtrees which split the original tree. \noindent The following is an important consequence of Proposition \ref{prop:indep_bimarginal_M}.

\begin{corollary}\label{prop:tree_split_to_edges}
    With $\mathcal{T}$ as in Proposition \ref{prop:indep_bimarginal_M},
    we have
\begin{equation}\label{eq:M_decomp_pairwise}
        [(\bm{M}^{\rm{opt}}_{\mathcal{T}})_{i_1,\dots,i_s}] = \frac{\prod_{(\sigma_1,\sigma_2)\in\mathcal{E}}[(M_{\sigma_1\sigma_2}^{\rm{opt}})_{i_{\sigma_1},i_{\sigma_2}}]}{\prod_{\sigma\in\llbracket s\rrbracket}(\bm{\mu}_\sigma)_{i_\sigma}^{\deg(\bm{\mu}_\sigma)-1}}\:,
    \end{equation}
    where ${\mathrm{deg}}$ denotes the degree of a vertex. Furthermore, let
\begin{align}
{\emph{\texttt{SB}}}_{\eta}\left(\bm{\mu}_{\sigma_{1}},\bm{\mu}_{\sigma_{2}}\right)\!:=\! {\mathrm{D_{KL}}}\!\left({M}^{\rm{opt}}_{\sigma_1\sigma_2}\!\parallel\! {K}_{\sigma_1\sigma_2}\right)
\label{defSBoptval}    
\end{align}
denote the (scaled) optimal value for the bimarginal SB problem. Then,
    \begin{align}
        &{\mathrm{D_{KL}}}\left(\bm{M}^{\rm{opt}}_{\mathcal{T}}\!\parallel \!\bm{K}\right) \!= \!\sum_{(\sigma_1,\sigma_2)\in\mathcal{E}}\!\!\!\!{\emph{\texttt{SB}}}_{\eta}\left(\bm{\mu}_{\sigma_{1}},\bm{\mu}_{\sigma_{2}}\right)\nonumber\\
        &\qquad\qquad\qquad+ \!\!\!\!\sum_{\sigma\in\llbracket s\rrbracket}\!\!(\deg(\bm{\mu}_\sigma)-1)H(\bm{\mu}_\sigma). \label{eq:KL_decomp_pairwise}
    \end{align}
\end{corollary}
\begin{proof}
    Follows by recursive application of Proposition \ref{prop:indep_bimarginal_M}.
\end{proof}

In the following, we use Corollary \ref{prop:tree_split_to_edges} to design a tractable algorithm for solving problem \eqref{OptimalGproblem}. 


\subsection{Construction of MST}\label{subsec:MSTconstrcution}
Corollary \ref{prop:tree_split_to_edges} allows for the decomposition of a tree along its constituent edges -- we may solve the MSB over a \emph{given} tree by solving the bimarginal SBPs over all edges of that tree and combining the solutions by \eqref{eq:M_decomp_pairwise}. Similarly, we may evaluate the cost of that tree by \eqref{eq:KL_decomp_pairwise} \emph{without} reconstructing $\bm{M}^{\rm{opt}}_{\mathcal{T}}$.

To instead \emph{construct} the optimal tree, we must get rid of the $\deg(\cdot)$ term in \eqref{eq:KL_decomp_pairwise}, which requires knowledge of the complete tree structure. With this aim, $\forall(\sigma_1,\sigma_2)\in\mathcal{E}$, we let 
\begin{align}
    &g_{\sigma_1\sigma_2} := {{\texttt{SB}}}_{\eta}\left(\bm{\mu}_{\sigma_{1}},\bm{\mu}_{\sigma_{2}}\right) + H(\bm{\mu}_{\sigma_1}) + H(\bm{\mu}_{\sigma_2}),    \label{eq:defgsigma}
\end{align}
and rewrite \eqref{eq:KL_decomp_pairwise} as
\begin{equation}\label{eq:KL_decomp_pairwise_independent}
{\mathrm{D_{KL}}}\left(\bm{M}^{\rm{opt}}_{\mathcal{T}}\parallel \bm{K}\right) = \!\!\!\!\!\sum_{(\sigma_1,\sigma_2)\in\mathcal{E}}g_{\sigma_1\sigma_2} - \sum_{\sigma\in\llbracket s\rrbracket} H(\bm{\mu}_\sigma).
\end{equation}
As $\sum H(\bm{\mu}_\sigma)$ is independent of the tree structure, we may take $g_{\cdot\cdot}$ as the truly additive `costs' of our edges. 
Notice from \eqref{eq:defgsigma} that $g_{\sigma_1\sigma_2}$ is symmetric in $\bm{\mu}_{\sigma_1},\bm{\mu}_{\sigma_2}$. 

We then rewrite our the optimal MSB problem \eqref{OptimalGproblem} as 
\begin{align}
&\mathcal{T}^{\mathrm{opt}} =\underset{\mathcal{E}\subset\mathcal{V}\times\mathcal{V}}{\arg\min}\:\sum_{(\sigma_1,\sigma_2)\in\mathcal{E}}g_{\sigma_1\sigma_2}.
\label{OptimalTreeProblem}    
\end{align}
Recall that by Proposition \ref{prop:Goptisatree}, the minimizer of \eqref{OptimalTreeProblem} must indeed be a tree. Thus \eqref{OptimalGproblem} $\equiv$ \eqref{OptimalTreeProblem} is \emph{exactly} the MST problem over the complete graph on $s$ vertices $\mathcal{V}$, with $g_{\cdot\cdot}$ defining the costs of the $s(s-1)/2$ unique edges.

Building on these results, we propose Algorithm \ref{alg:MSBasMST} to solve our original problem \eqref{OptimalGproblem}. 

\begin{algorithm}
\caption{Optimal MSB as an MST Problem}\label{alg:MSBasMST}
\begin{algorithmic}
\Require Distribution set $\mathcal{V} \gets \{\bm{\mu}_\sigma\}_{\sigma\in\llbracket s\rrbracket}$, edge set $\mathcal{E}$ of the complete graph over $s$ vertices, ground cost function $c$, entropic regularization parameter $\eta>0$, bimarginal Sinkhorn algorithm \texttt{AlgSINK}, MST algorithm \texttt{AlgMST}.\newline

\For{$(\sigma_1,\sigma_2)\in\mathcal{E}$}
\State $C_{\sigma_1\sigma_2} \gets c(\bm{x}_{\sigma_1},\bm{x}_{\sigma_2})\:\forall(\bm{x}_{\sigma_1},\bm{x}_{\sigma_2})\in\mathcal{X}_{\sigma_1}\times\mathcal{X}_{\sigma_2}$ 
\State ${{\texttt{SB}}}_{\eta}\left(\bm{\mu}_{\sigma_{1}},\bm{\mu}_{\sigma_{2}}\right) \gets \texttt{AlgSINK}(C_{\sigma_1\sigma_2},\eta,\bm{\mu}_{\sigma_1},\bm{\mu}_{\sigma_2})$
\State $g_{\sigma_1\sigma_2} \gets \eqref{eq:defgsigma}$ \Comment{edge weights}
\EndFor
\State $\mathcal{T}^{\rm{opt}} \gets \texttt{AlgMST}(\mathcal{V},\mathcal{E},\{g_{\sigma_1\sigma_2}\}_{(\sigma_1,\sigma_2)\in\mathcal{V}\times\mathcal{V}})$
\end{algorithmic}
\end{algorithm}

Notice that Algorithm \ref{alg:MSBasMST} requires two subroutines: \texttt{AlgSINK} (bimarginal Sinkhorn recursion\footnote{This is recursion \eqref{MultimargSinkhornSolution} with $s=2$.} to solve the SB problem) and \texttt{AlgMST} (standard MST algorithm such as Dijkstra-Jarník-Prim algorithm or Borůvka's algorithm \cite[Sec. 2]{pettie2002optimal}). The \textbf{for} loop in Algorithm \ref{alg:MSBasMST} constructs the edge weight matrix $g_{\sigma_1\sigma_2}$, and with those edge weights, we compute the MST for the complete weighted graph over $\mathcal{V}$. 

We state Theorem \ref{thm:finalresult} as a summary of the above.

\begin{mycolortheorem}\label{thm:finalresult}
Let the tree $\mathcal{T}^{\mathrm{opt}}$ be the output of Algorithm \ref{alg:MSBasMST} for input $\mathcal{V}:=\{\bm{\mu}_\sigma\}_{\sigma\in\llbracket s\rrbracket}$. Then $\mathcal{T}^{\mathrm{opt}}$ solves the optimal MSB problem, i.e., $\mathcal{T}^{\mathrm{opt}}=\mathcal{G}^{\mathrm{opt}}$ in problem \eqref{OptimalGproblem}. 
\end{mycolortheorem}

\begin{remark}[Chow-Liu tree] The proposed Algorithm \ref{alg:MSBasMST} has structural similarity with the Chow-Liu algorithm \cite{chow1968approximating} for optimal approximation of an arbitrary discrete probability measure by a product of second-order (conditional and marginal) measures, or equivalently approximating a Markov random field by a first-order dependency tree, w.r.t. the ${\mathrm{D_{KL}}}$ loss. The Chow-Liu algorithm solves this problem by constructing a complete weighted graph with edge weights being pairwise mutual information \cite[Ch. 2.3]{cover1991elements},
and then computing the maximum weight spanning tree. Apart for being the solution to a different problem, our edge weights comprise $g_{\sigma_1\sigma_2}$ in \eqref{eq:defgsigma}, and we compute the minimum (not maximum) weight spanning tree. 
\end{remark}

\begin{remark}[Existence-uniqueness of $\mathcal{T}^{\mathrm{opt}}$] For problem \eqref{OptimalGproblem}, the existence of $\mathcal{T}^{\mathrm{opt}}$ is guaranteed since every non-empty finite set (here a set of spanning trees of cardinality $s^{s-2}$) contains the extrema. Akin to a general MST, if all edge weights (here $g_{\sigma_1\sigma_2}$) are unique, then $\mathcal{T}^{\mathrm{opt}}$ is unique.      
\end{remark}

\begin{remark}[Parallelization] Since computing $g_{\sigma_1\sigma_2}$ for one edge is decoupled from the other, the \textbf{for} loop in Algorithm \ref{alg:MSBasMST} can be parallelized. 
\end{remark}


\subsection{Computational Complexity for Algorithm \ref{alg:MSBasMST}}\label{subsec:complexity}
For simplicity, let us consider $n_{\sigma}=n$ $\forall\sigma\in\llbracket s\rrbracket$. The complexity of Algorithm \ref{alg:MSBasMST} comprises those of its two steps: constructing the edge weights, and solving the MST.

The complexity analyses\footnote{These results only assume $\bm{C}\geq 0$ elementwise.} for solving an instance of the bimarginal SB via Sinkhorn recursion (i.e., \eqref{MultimargSinkhornSolution} with $s=2$) consider the notion of $\varepsilon$-accurate solution: a coupling $\widehat{\bm{M}}\in\Pi(\mathcal{V})$ is called \emph{$\varepsilon$-accurate} for some $\varepsilon>0$ if $\langle\bm{C},\widehat{\bm{M}}\rangle \leq \langle\bm{C},\bm{M}^{\mathrm{opt}}\rangle + \varepsilon$.
The best known \cite{dvurechensky2018computational} complexity is $\mathcal{O}\left(n^2\|\bm{C}\|_{\infty}^{2}\log n/\varepsilon^2\right)$, i.e., $\widetilde{O}\left(n^2/\varepsilon^2\right)$, which improves upon the earlier \cite{altschuler2017near} $\mathcal{O}\left(n^2\|\bm{C}\|_{\infty}^{3}\log n/\varepsilon^3\right)$, i.e., $\widetilde{O}\left(n^2/\varepsilon^3\right), \eta= \varepsilon/(4\log n)$. Since there are $s(s-1)/2$ edges in a complete graph, the complexity for solving the bimarginal SBs for all edges is $\mathcal{O}\left(s^2 n^2\|\bm{C}\|_{\infty}^{2}\log n/\varepsilon^2\right)$, which remains the dominant complexity in computing \eqref{eq:defgsigma} for all edges.

For the weighted complete graph thus constructed, the complexity of computing the MST via the Dijkstra-Jarník-Prim algorithm or Borůvka's algorithm is $\mathcal{O}\left(s^2\right)$ \cite[Sec. 5]{fredman1987fibonacci}, \cite[Sec. 2.1]{eppstein2000spanning}. We note here that  Kruskal's algorithm \cite[Ch. 23.2]{cormenbook3rded}--another common algorithm for computing the MST--has a larger $\mathcal{O}\left(s^2\log s\right)$ complexity for complete graphs due to its sorting of all edge weights.

Therefore, the total complexity for Algorithm \ref{alg:MSBasMST} is $\mathcal{O}\left(s^2 n^2\|\bm{C}\|_{\infty}^{2}\log n/\varepsilon^2\right)$. In terms of the regularizer $\eta$, this complexity is $\mathcal{O}\left(s^2 n^2\|\bm{C}\|_{\infty}^{2}\left(\log n\right)^{-1}/\eta^2\right)$.


\begin{figure}[t]
    \centering
\includegraphics[width=0.95\linewidth]{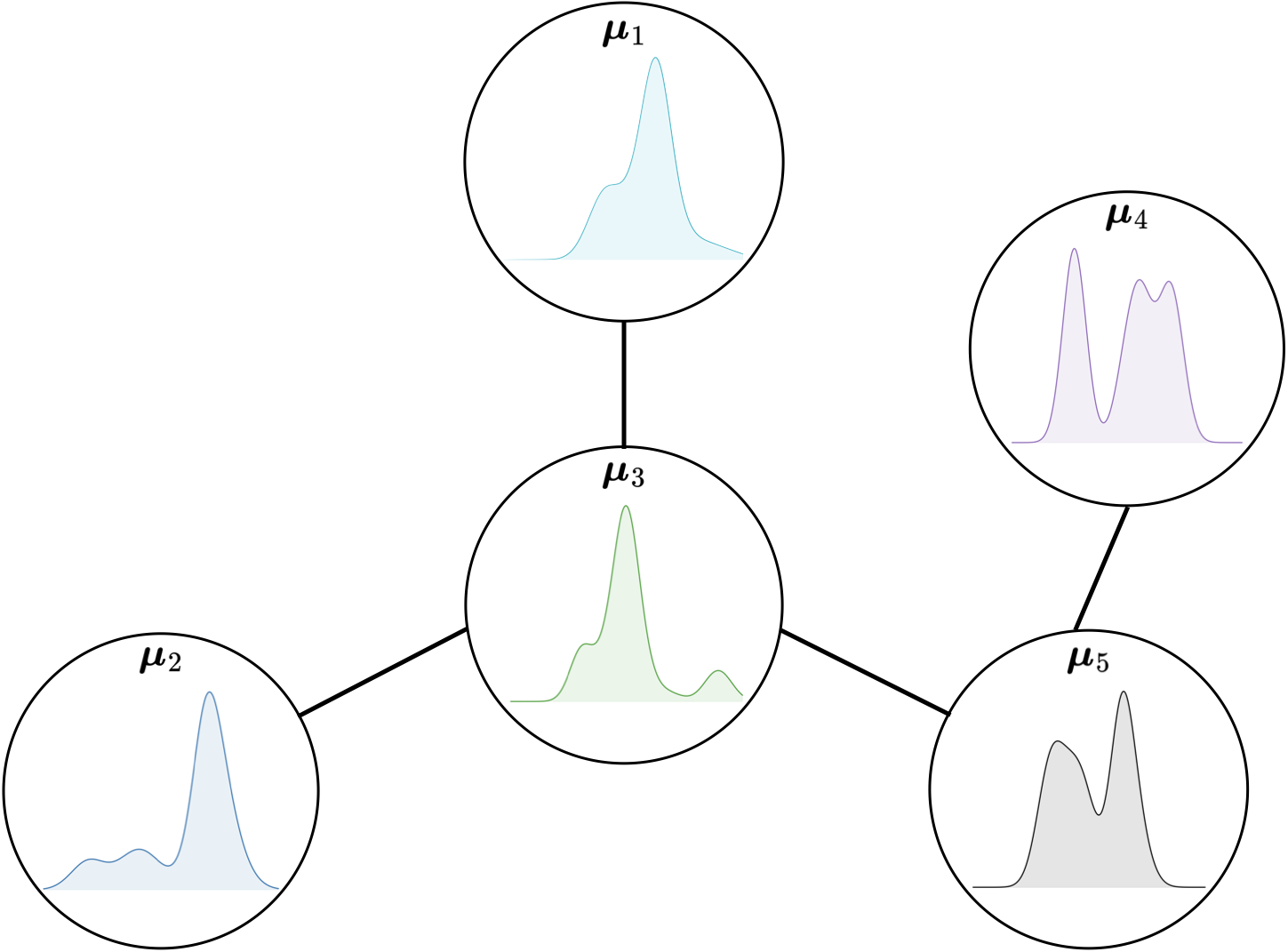}
    \caption{The optimal MSB graph structure for the numerical experiment in Sec. \ref{subsec:GMMs}. The densities for the measure-valued vertices are shown inside the circles.}
    \vspace*{-0.2in}
    \label{fig:GMM}
\end{figure}

\section{Numerical Results}\label{sec:NumericalResults}
To demonstrate our solution of the optimal MSBP, outlined in Sec. \ref{sec:TwoStepSolution}, we perform two numerical experiments. In the first, we compare the results and performance of Algorithm \ref{alg:MSBasMST} to na\"ive computation of spanning trees over a small set of Gaussian mixture-sampled vertices. In the second, we apply Algorithm \ref{alg:MSBasMST} for event reconstruction over video frames. All experiments were done using MATLAB R2024b on a Debian 12 Linux machine with an AMD Ryzen 7 5800X CPU. For MST computation with the Dijkstra-Jarník-Prim algorithm, we used the MATLAB command \texttt{minspantree(G,'Method','dense')}.

\subsection{Optimal MSB over Gaussian Mixtures}\label{subsec:GMMs}
We take $n=25$ samples from each of the $s=5$ Gaussian mixtures supported on $[-10,10]$ in Fig. \ref{fig:GMM}, to form our vertices $\{\bm{\mu}_\sigma\}_{\sigma\in\llbracket 5\rrbracket}$. We then solve for the optimal MSB first by constructing all $s^{s-2}=125$ possible spanning trees and comparing their MSB costs (as in the objective of \eqref{OptimalGproblem}), and then by construction from pairwise SB as in Algorithm \ref{alg:MSBasMST}. The computed MST shown in Fig. \ref{fig:GMM} has Pr\"{u}fer code: $3\:3\:5$, and is not a path. Table \ref{tab:GMMtrees_comp} shows the comparison of the $10$ lowest-cost spanning trees on these measure-valued vertices.

\begin{table}[h]
\centering
\begin{tabular}{| c | c | c |} 
 \hline
 ~Pr\"ufer code for $\mathcal{T}$~ & Cost \eqref{OptimalGproblem} & Cost \eqref{eq:KL_decomp_pairwise_independent}\\
 \hline\hline
 \:\:3\:3\:5\:\: & \:\:0.279922072756287\:\: & \:\:0.279921946258293\:\: \\
 \hline
 3\:3\:4 & 0.295777099784325 & 0.295776598462776 \\
 \hline
 3\:3\:3 & 0.316798945466359 & 0.316798407793688 \\ 
 \hline
 5\:3\:5 & 0.317890033393032 & 0.317889978780085 \\
\hline
 3\:5\:5 & 0.319564634628009 & 0.319562860510066 \\
\hline
 4\:3\:5 & 0.320487973256691 & 0.320487070929787 \\
\hline
 3\:2\:5 & 0.325540126207581 & 0.325538361972965 \\
\hline
 5\:3\:4 & 0.333745066643530 & 0.333744630984569 \\
\hline
 3\:5\:4 & 0.335419414386392 & 0.335417512714549 \\
\hline
 4\:3\:4 & 0.336342669176185 & 0.336341723134270 \\
\hline
\end{tabular}
\caption{{\small{Pr\"ufer codes and MSB costs of the $10$ lowest-cost trees out of the total $125$ trees, computed both globally and edgewise.}}}
\label{tab:GMMtrees_comp}
\vspace*{-0.1in}
\end{table}

\begin{figure*}[t]
    \centering
\includegraphics[width=0.9\linewidth]{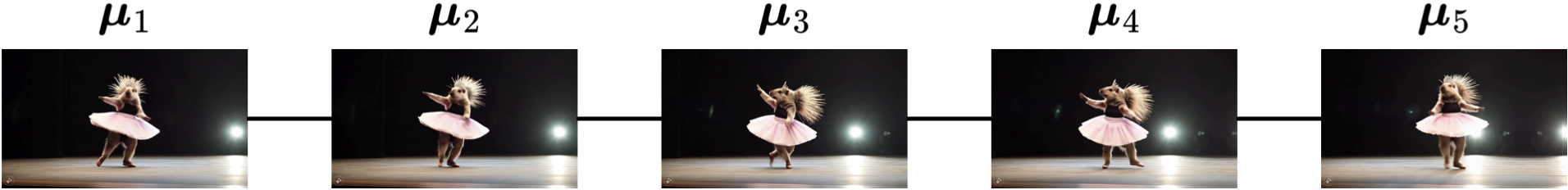}
    \caption{The optimal MSB graph structure for the numerical experiment in Sec. \ref{subsec:video}.}
    \vspace*{-0.2in}
    \label{fig:procupine}
\end{figure*}

We note that the proposed Algorithm \ref{alg:MSBasMST} took $\approx0.226$ seconds to construct the optimal tree, along with all other possible trees. Solving over each possible spanning tree, however, took $\approx3$ minutes per tree. Table \ref{tab:GMMtrees_comp} shows that both approaches found the same MST, and the corresponding tree costs match well.

\subsection{Spatio-temporal Reconstruction}\label{subsec:video}
To demonstrate that the proposed formulation and its solution can enable most-likely spatio-temporal reconstruction, we generated $s=5$ frames from the generative AI video \texttt{Porcupine}\footnote{URL: \url{https://www.youtube.com/watch?v=cCjegaOq2hQ}} generated by prompt: ``A porcupine wearing a tutu, performing a ballet dance on a stage" using the Movie Gen \cite{moviegenresearchpaperURL} cast of foundation models. We downsampled these frames to $128\times128$ pixels, and used them without timestamps\footnote{The timestamps for the frames are $\{0.166,0.233,0.333,0.666,0.866\}$ seconds, for $\sigma=1,2,3,4,5$, respectively.} as the vertices $\{\bm{\mu}_{\sigma}\}_{\sigma\in\llbracket 5\rrbracket}$. Because of joint spatio-temporal correlation among these frames, the optimal MSB graph structure is not obvious only from these snapshots. 

Despite the intentional deletion of timestamps, Algorithm \ref{alg:MSBasMST} applied to vertex data $\{\bm{\mu}_{\sigma}\}_{\sigma\in\llbracket 5\rrbracket}$ found the optimal MSB graph structure to be the path tree shown in Fig. \ref{fig:procupine}, i.e., the path is in the same sequence as frame capture. In our platform mentioned earlier, Algorithm \ref{alg:MSBasMST} incurred $\approx37$ minutes computational time for this data.


\section{Conclusion}\label{sec:Conclusion}
We formulate the optimal MSB problem and show its equivalence to finding an MST over a given set of measure-valued vertices. We then derive a tree decomposition result and using it, propose an algorithm to compute this MST. We report two numerical experiments to illustrate our results.


\bibliographystyle{IEEEtran}
\bibliography{references}

\end{document}